\documentclass[conference]{IEEEtran}
\IEEEoverridecommandlockouts

\usepackage{cite, hyperref}
\usepackage{amsmath,amssymb,amsfonts}
\usepackage{float}
\usepackage{algorithm}
\usepackage{algpseudocode}
\usepackage{graphicx}
\usepackage{textcomp}
\usepackage{xcolor}
\usepackage{amsthm}
\newtheorem{theorem}{Theorem}[]

\newtheorem{lemma}{Lemma}[]
\newtheorem{proposition}{Proposition}[]
\newtheorem{remark}{Remark}[theorem]
\newtheorem{definition}{Definition}
\newcommand{\remove}[1]{}
\def\BibTeX{{\rm B\kern-.05em{\sc i\kern-.025em b}\kern-.08em
    T\kern-.1667em\lower.7ex\hbox{E}\kern-.125emX}}
\begin{document}

\title{Sequence prediction under a lying oracle\\
}

\author{\IEEEauthorblockN{Puspabeethi Samanta}
\IEEEauthorblockA{\textit{Department of Electrical Engineering} \\
\textit{IIT Bombay}\\
puspabeethi.samanta@gmail.com}
\and
\IEEEauthorblockN{Nikhil Karamchandani}
\IEEEauthorblockA{\textit{Department of Electrical Engineering} \\
\textit{IIT Bombay}\\
nikhilk@ee.iitb.ac.in}
\and
\IEEEauthorblockN{Jayakrishnan Nair}
\IEEEauthorblockA{\textit{Department of Electrical Engineering} \\
\textit{IIT Bombay}\\
jayakrishnan.nair@ee.iitb.ac.in}
}

\maketitle

\begin{abstract}
We consider the problem of sequential prediction of an $m$-ary sequence, where at each epoch, (i) the environment selects an outcome from an $m$-ary alphabet, (ii) the learner selects a probability distribution over the same alphabet (unaware of the outcome generated by the environment), and finally, (iii) the learner incurs a cost that depends on the probability assigned to the outcome. The cost function we consider captures the complexity of predicting the outcome generated by the environment, in a scenario where the aforementioned prediction is performed via comparative queries to a lying oracle. We consider both stochastic and adversarial environments, propose algorithms for both settings, and establish logarithmic upper bounds on their regret.
\end{abstract}


\section{Introduction}
\label{Sec:Intro}

Sequential prediction of discrete outcomes is a central problem in information theory, statistics, and online learning. In this setting, at each round $t$, a learner predicts a probability distribution $p_t$ over a finite alphabet, and then observes an outcome $y_t$. The performance of the learner is measured through a loss function $l(p_t,y_t)$, and different choices of loss functions lead to different notions of optimality and different algorithmic behaviors. Common examples include logarithmic loss, squared loss, and absolute loss \cite{b7,vovk2001game,gneiting2007strictly}. 

Among these, the logarithmic loss $\log \frac{1}{p_t(y_t)}$ plays a particularly prominent role due to its connections to coding, gambling, and universal prediction. In particular, cumulative log-loss corresponds to the codelength of the observed sequence under the assigned probabilities, and minimizing regret under log-loss is equivalent to achieving optimal redundancy in universal coding \cite{b5,b6,b7}. This connection has led to a rich body of work on sequential probability assignment, including classical estimators such as the add-$\beta$ family, which includes the Laplace ($\beta=1$) and the Krichevsky--Trofimov  ($\beta=1/2$) estimators, and which achieve optimal or near-optimal regret under log-loss \cite{b6,b7,b8,b10}.

An alternative interpretation of log-loss arises from the classical twenty questions problem \cite{renyi1961,cover2012elements}, a sequential game for identifying an unknown symbol by querying an oracle. In the noiseless setting, if a symbol $x$ is drawn from a distribution~$\mu$, then the number of Yes/No queries required to identify~$x$ scales, up to lower-order terms, as $\log (1/\mu(x))$, and the expected number of queries is given by the entropy $H(\mu)$. This establishes a direct connection between optimal querying strategies and sequential probability assignment under log-loss, when the underlying distribution is a priori unknown \cite{b3}.

A natural extension of this problem, known as the R\'enyi--Ulam game \cite{b11}, allows the oracle to lie a bounded number of times. In this setting, the presence of lies increases the query complexity required to identify the unknown symbol. The survey paper \cite{pelc} of Pelc provides a detailed review of the various settings of this game in the presence of lies/errors. Recent results show that allowing up to $k$ lies introduces an additional overhead of order $k \log \log (1/\mu(x))$ in the number of queries, beyond the classical $\log (1/\mu(x))$ term \cite{b1,b11}. 

Motivated by this connection, we consider the following perturbed loss function. At each round $t$, after predicting a distribution $p_t$ and observing outcome $y_t$, the learner incurs loss
\[
l(p_t,y_t)
=
\log \frac{1}{p_t(y_t)}
+
k \log \log \left( \frac{1}{p_t(y_t)} + c \right)
\]
where $k \geq 0$ represents the number of allowed lies and $c \geq 1$ is a constant. The first term corresponds to the classical log-loss, while the second term captures the additional complexity induced by unreliable oracle responses.

We study sequential probability assignment under this perturbed loss in both stochastic and adversarial settings. In the stochastic setting, outcomes are generated i.i.d.\ from an unknown distribution $p$, and performance is measured through the expected \textit{regret}, i.e., the cumulative loss relative to predicting the true distribution $p$ at every round. We first analyze the performance of the add-$\beta$ family of estimators under this loss, which as we mentioned before, is central to the theory of universal prediction and achieves optimal or near-optimal regret under log-loss \cite{b6,b7,b8,b10}. 

We then turn to the adversarial setting, where the outcome sequence is arbitrary and performance is measured relative to the best fixed distribution in hindsight. To address this setting, we consider the Exponentially Weighted Online Optimizer (EWOO) \cite{b4, b7}. We show that the perturbed loss function is exp-concave, which allows us to apply EWOO and obtain regret guarantees.

Finally, we complement our theoretical analysis with numerical experiments that illustrate the behavior of the proposed loss and algorithms, and highlight the dependence of regret on the horizon $n$, the parameter $k$ as well as the choice of estimator. 

\section{Problem Formulation}

\subsection{General Model}
We consider the problem of sequential prediction of $m$-ary sequences. More concretely, in each round $t \in [n] := \{1,2,\ldots,n\}$, the environment generates a symbol $y_t \in \mathcal{Y} = \{1,2,\ldots,m\}$, which is apriori unknown to the learner. Based on the history of prior outcomes $y_1,y_2, \ldots, y_{t-1}$, the learner is asked to assign a probability mass function $p_t = (p_t(1), p_t(2),\ldots,p_t(m))$ over $\mathcal{Y}$. The outcome $y_t$ is then revealed to the learner, who incurs a loss $l(p_{t}, y_{t})$. This marks the completion of one round of the learning problem, and this process is continued similarly for subsequent rounds.

\subsection{Stochastic regret}
In the stochastic setting, the environment generates $y_{t}$ in an i.i.d. fashion as per a fixed distribution $p$ (apriori  unknown to the learner) over $\mathcal{Y}, \forall \ t \in [n]$. At the end, the performance of the learner is measured in terms of the cumulative expected \textit{regret} which is given by 

\begin{align}
    R_{n}^{S} \ &\triangleq \ \sum_{t=1}^{n} E_{y_{t} \sim p}\ E_{y_{\tau}|_{\tau=1}^{t-1}}\ [l(p_{t}, y_{t}) -  l(p, y_{t})] \\
    &= \sum_{t=1}^{n}\sum_{j=1}^{m}\  p(j)E_{y_{\tau}|_{\tau=1}^{t-1}}\  [l(p_{t}, j) -  l(p, j)]
\end{align}

\subsection{Worst-case regret}
In this setting, the outcome sequence $y_1,y_2,\ldots,y_n$ is arbitrary (and can even be adversarially determined). For each such sequence, the regret of the learner is measured by the gap between its cumulative incurred loss and that of the `best' distribution in hindsight, i.e., 
\begin{align}
     R_{n}^{A}(y_{t}|_{t=1}^{n}) &\triangleq \sum_{t=1}^{n} l(p_{t}, y_{t}) - \underset{p \in \mathcal{D}}{\textnormal{min}}\sum_{t=1}^{n} l(p, y_t) \\
     &=   \sum_{t=1}^{n}\ [l(p_{t}, y_{t}) - l(p^{*}(y_{t}|_{t=1}^{n}), y_{t})]
 \end{align}
 where $\mathcal{D}$ denotes the set of all $m$-ary probability distributions and $p^{*}(y_{t}|_{t=1}^{n}) = \underset{p \in \mathcal{D}}{\textnormal{arg min}}\ \sum_{t=1}^{n} l(p, y_{t})$. 
The overall regret of the learner is then defined by taking a supremum over all possible sequences in $\mathcal{Y}^{n}$, as detailed next.
\begin{align}
     R_{n}^{A} &\triangleq \underset{y_{t}|_{t=1}^{n} \in \mathcal{Y}^{n}}{\textnormal{sup}}R_{n}^{A}(y_{t}|_{t=1}^{n})\\
    &= \underset{y_{t}|_{t=1}^{n} \in \mathcal{Y}^{n}}{\textnormal{sup}}\sum_{t=1}^{n}\ [l(p_{t}, y_{t}) - l(p^{*}(y_{t}|_{t=1}^{n}), y_{t})]
\end{align}

\subsection{Loss Function}
As highlighted in the introduction, in this paper we focus on a particular loss function which is motivated by the problem of guessing with lies \cite{b1}, given as follows for any $p \in \mathcal{D}, y \in~[m]$:
\begin{equation}
    l(p, y) := f(p(y)) = \log\frac{1}{p(y)} + k\log\log\left(\frac{1}{p(y)}+c\right)
    \label{eqn:lossfn}
\end{equation}
where $k \geq 0, c \geq 1$ are constants; $k$ is an integer, and $c$ a real number. Note that for $k=0$, this matches the widely studied log loss function \cite{b5}.

\section{Stochastic Setting}

For this setting, we consider the popular \textit{add-$\beta$ estimator} \cite{b6,b7},
which coincides with the Bayesian posterior  distribution under a symmetric
Dirichlet prior. It can also be expressed as the convex combination of a uniform prior and the empirical estimate, as shown below.
\begin{equation}
    p_{t}(j) = \frac{\beta + N_{j, t-1}}{t-1+m\beta} = (1 - \lambda_{t})\frac{1}{m} + \lambda_{t}\bar{p}_{j,t-1}, \ \forall \ j \in [m]
    \label{eqn:addbeta}
\end{equation}
with $\lambda_{t} = 1 - \frac{m\beta}{t-1+m\beta}$ and $\bar{p}_{j, t-1}  = \frac{N_{j, t-1}}{t-1}$ where $N_{j,t-1}$ denotes the number of occurrences of symbol $j$ in
$y_1,\dots,y_{t-1}$.

The following theorem gives a logarithmic upper bound on the cumulative regret $R_{n}^{S}$ incurred by the add-$\beta$-estimator over~$n$ rounds in the stochastic setting.

\begin{theorem}
    For the sequential prediction problem of an $m$-ary sequence $(y_{1}, \cdots, y_{n}) \in \mathcal{Y}^{n}$, where each symbol $y_{t}, t \in [n]$ is generated i.i.d. as per a fixed distribution $p$, and the loss function is as specified in \eqref{eqn:lossfn}, the add-beta estimator incurs logarithmic cumulative regret given by 
\begin{equation*}\begin{split}
R_{n}^{S} \leq \Bigg(m\beta \ \textnormal{log} \ m + \frac{4m}{\ln 2} + & \frac{k}{(\ln 2)\ l} \left(m^{2}\beta + \frac{8m}{ p_{min}}\right)\Bigg) \\ & \ln \bigg( \frac{n-1}{\lfloor m\beta \rfloor}+1\bigg) + O(1)
\end{split}
\end{equation*}
    where $l = \left(\frac{1}{p_{max}} + c\right)\textnormal{ln}\left(\frac{1}{p_{max}} + c\right),  
 p_{max} = \underset{j \in \mathcal{Y}}{\textnormal{max}}\ p(j),$ $ \textnormal{and} \ p_{min} = \underset{j \in \mathcal{Y}}{\textnormal{min}}\ p(j)$.
\end{theorem}

The proof of Theorem 1 is documented in Section VII-A. Note that the bound scales logarithmically with the time horizon $n$, and linearly with the support size $m$ of the underlying distribution. This is identical to the regret behaviour for the widely studied log-loss function \cite{b5,b6}, which corresponds to $k =0$ in the loss model studied here. Also, note that the regret grows linearly with $k$, which represented the number of permitted lies in the query model mentioned in Section~\ref{Sec:Intro}, which served as a motivation for the loss function studied here.

\section{Adversarial Setting}
In this section, we consider the setting where the outcome sequence is arbitrary. For this setting, we employ an estimator called the \emph{Exponentially Weighted Online Optimizer (EWOO)}, which is known to achieve a cumulative regret of $O(\log n)$ when the  loss function is exponentially concave \cite{b4}. This estimator generates probability estimates by taking a weighted average across the simplex, wherein the weights are assigned based on the exponential of the cumulative loss in hindsight.

We begin by showing that the loss function $f(p(y_{t}))$, as given in section II-D, is not only a convex function in $p(y_{t}) \in (0, 1]$, but also an $\alpha$-exponentially-concave function for a suitable choice of $\alpha$.\\

\begin{definition}
    A convex function $f : \mathcal{X} 
\rightarrow \mathcal{R}$
is $\alpha$-exp-concave over 
$\mathcal{X} \subseteq \mathcal{R}$
if the function $g : \mathcal{X} \rightarrow \mathcal{R}$ defined as $g(x) = e^{-\alpha f(x)}$ is concave.
\end{definition}

\begin{theorem}
\label{Thm2}
Let 
$k \geq 0$, and $c \geq 1$. Then the convex function $f : (0, 1] \rightarrow \mathcal{R}$ defined as $f(x) = \log \frac{1}{x} + k \log \log \left(\frac{1}{x} + c\right)$ is $\frac{a^{*}}{\log_{2} e}$-exp-concave,
where $a^{*} :=$
\begin{equation} 
    \!\underset{x \in (0, 1]}{\textnormal{min}}\ \frac{(1+xc)^{2}(\ln(\frac{1}{x}+c))^{2} + k(1+2xc)\ln(\frac{1}{x}+c) - k}{(1+xc)^{2}(\ln(\frac{1}{x}+c))^{2} + 2k(1+xc)\ln(\frac{1}{x}+c)+k^{2}}.
    \label{Eqn:astar}
\end{equation} 
Alternately, a simpler form of the result can be obtained, which states that f(x) is $\frac{a^{**}}{\log_{2}e}$-exp-concave, where 
\begin{equation}
    a^{**} = 1 - \frac{k}{\gamma(c)+k} - \frac{k}{(\gamma(c)+k)^{2}}
\end{equation}
with $\gamma(c) = \underset{x \in (0, 1]}{\textnormal{min}}\left[(1+xc)\ln\big(\frac{1}{x}+c\big)\right]$. 
\end{theorem}

\begin{remark}
\label{remark1}
It can be argued that $$a^* \ge a^{**} \ge (2\ln 2 - 1)/ k.$$ This suggests that the exp-concavity constant for the loss function $f(x)$ scales inverse-proportionally with respect to the parameter~$k.$
\end{remark}

The proof of Theorem~\ref{Thm2} can be found in the Appendix.

Algorithm~\ref{Alg:1} describes the EWOO scheme, which at each time step outputs a probability assignment given by a weighted mixture over all $m$-ary discrete distributions, where the weights are proportional to an exponential function of their cumulative loss in hindsight. In particular, we assign the following weight $w_{t}(p)$ to distribution $p \in \mathcal{D}$ in round $t$, using the value of $\alpha$ from the above theorem.
\begin{align}
w_{t}(p) =
  \begin{cases}
    1 \ \ \ \ \ \ \ \ \ \ \ \ \ \ \ \ \ \ \ \ \ \ \ \text{for} \ \ t = 1\\
    e^{-\alpha \sum_{\tau = 1}^{t-1}f(p(y_{\tau}))} \ \ \ \ \text{for} \ \ t > 1
  \end{cases}
\end{align}

These weights are then used to calculate $p_{t} \ \forall \ t \in [n]$, using Algorithm 1.

\begin{algorithm}
\label{Alg:1}
\caption{Exponentially Weighted Online Optimizer ($\alpha$)}
\begin{algorithmic}
\State \textbf{Input} : $k \geq 0, c \geq 1$, $\mathcal{Y} = \{1, \cdots, m\}$, convex set $\mathcal{X} = (0, 1]$, $\mathcal{D} = \{[p(1), \cdots, p(m)]^{T} : p(j) \in (0, 1] \ \forall \ j \in [m], \sum_{j=1}^{m}p(j) = 1\}$.\\
\State $\forall\ p \in \mathcal{D}$,
\begin{equation*}
    w_{1}(p) = 1
\end{equation*}
\begin{equation*}
    w_{t}(p) = e^{-\alpha \sum_{\tau = 1}^{t-1}f(p(y_{\tau}))} \ \forall\ t \geq 2
\end{equation*}
 \For{$t = 1, \cdots, n$}
 \State \begin{equation}
     p_{t} = \frac{\int_{p \in \mathcal{D}}p w_{t}(p) dp}{\int_{p \in \mathcal{D}}w_{t}(p) dp}
 \end{equation}
 \EndFor
\end{algorithmic}
\end{algorithm}

We finally obtain the following result on the worst-case regret performance of EWOO for our prediction problem.

\begin{theorem}
For the sequential prediction problem of an $m$-ary sequence $(y_{1}, \cdots, y_{n}) \in \mathcal{Y}^{n}$, where the loss function is as specified in \eqref{eqn:lossfn}, the EWOO estimator incurs logarithmic worst-case cumulative regret given by 
    \begin{equation}
        R_{n}^{A} \leq \frac{m}{a^{*}}\log_{2}n + \frac{2\log_{2}e}{a^{*}}
    \end{equation}
where $a^*$ is as given in \eqref{Eqn:astar}.
\end{theorem}

The regret bound in Theorem 3 is obtained by combining the regret performance of EWOO, given in Theorem 4.4 of \cite{b4}, and the results obtained from Theorem 2. From Remark~\ref{remark1}, $a^* \ge (2\ln 2 - 1)/ k$ and thus, the adversarial regret for EWOO scales as $O(mk\log n)$, similar to the stochastic regret for the add-$\beta$ estimator. Finally, while a naive implementation of EWOO can be computationally very expensive, more practical implementation techniques based on random sampling have been discussed in the literature \cite{kalai_vempala}.

\section{Simulation Experiments}
In this section, we conduct some numerical evaluations to support our theoretical results. 

\subsection{Implementations in Stochastic setting}
We consider a Bernoulli($p$) source ($m = 2$) generating the outcome sequence and compare the performance of the add-$\beta$ estimator (for $\beta = 1/2 , 1$ corresponding to Jeffrey's prior and the uniform prior respectively) for different values of $k$. 
Finally, we consider $n = 1000$ and report regret values which are averaged over $500$ sequences generated i.i.d according to the underlying Bernoulli distribution.

\begin{figure}[h!]
    \begin{center}
    \includegraphics[width=1\linewidth]{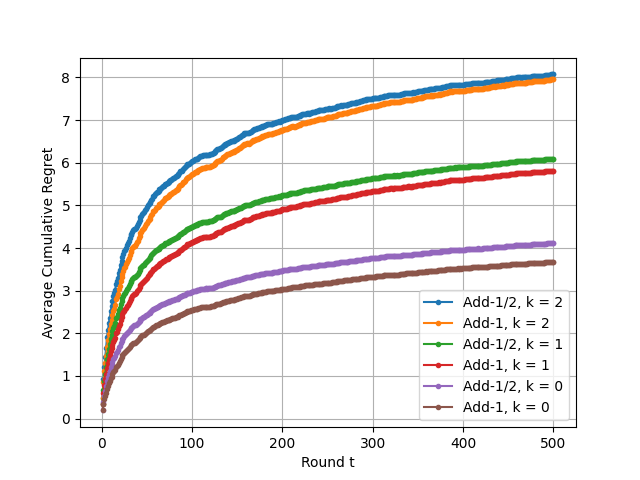}
        \caption{Cumulative regret under add-$\beta$ estimators, averaged over 1000 runs; outcomes are generated using $p = [0.25, 0.75]$.}
        \label{fig:1}
    \end{center}
\end{figure}

In Fig.~\ref{fig:1}, we plot the average cumulative regret of the different schemes for $k = 0, 1, 2$ under the Bernoulli(.25) source. In agreement with our theoretical results, we find that the regret scales logarithmically with the number of rounds. Note also that regret increases with increasing~$k.$

Next, we study the variation of the average cumulative regret of the add-$\beta$ estimator with the hyperparameter~$\beta$ in the stochastic setting. Recall that the add-$\beta$ estimator is given by $p_{t}(j) = \frac{m\beta}{t-1+m\beta}\frac{1}{m} + (1 - \frac{m\beta}{t-1+m\beta})\bar{p}_{j, t-1}$. Intuitively, a higher value of $\beta$ increases the weight of the uniform distribution (and therefore decreases the weight of the empirical distribution) on the actions  taken.

\begin{figure}[h!]
    \begin{center}
    \includegraphics[width=0.95\linewidth]{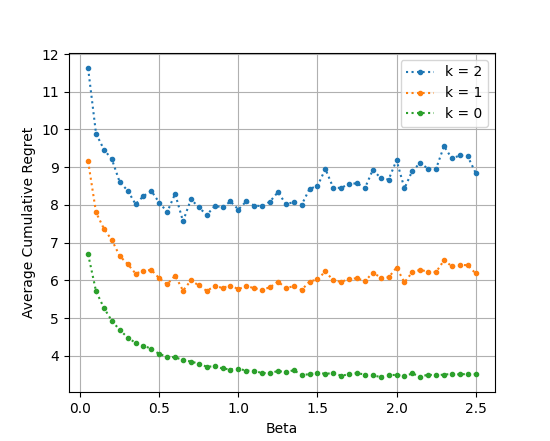}
        \caption{Variation of average cumulative regret of add-$\beta$ estimator with beta, for $p = [0.25, 0.75]$, 500 rounds, and 1000 samples.}
        \label{fig:2}
    \end{center}
\end{figure}

Fig. 2 shows the variation of average cumulative regret for 50 different beta values from 0.05 to 2.5 when the i.i.d. probability distribution is $p = [0.25,\ 0.75]$. We note that the optimal choice of~$\beta$ depends on the value of~$k$ (and indeed also on the horizon~$n;$ this latter dependence is not demonstrated here).
\begin{figure}[h!]
    \begin{center}
    \includegraphics[width=0.9\linewidth]{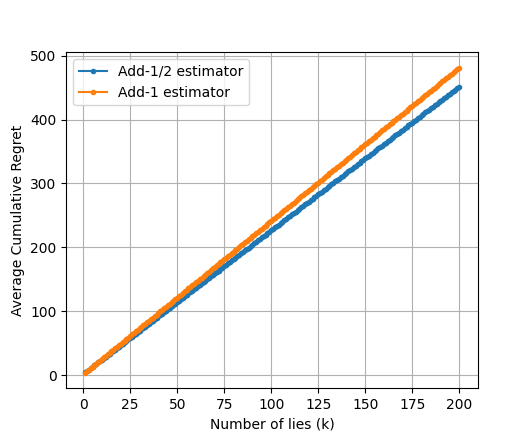}
        \caption{Variation of average cumulative regret of add-$\beta$ estimators with number of lies, for $p = [0.25, 0.75]$,  500 rounds, and 1000 samples.}
        \label{fig:3b}
    \end{center}
\end{figure}

Fig. 3. demonstrates that the average cumulative regret for add-beta estimator grows linearly with $k$, the number of lies in the query model, that matches with the expected dependence of the cumulative regret on $k$.

\subsection{Implementations in Adversarial setting}

\begin{figure}[h!]
    \begin{center}
    \includegraphics[width=0.9\linewidth]{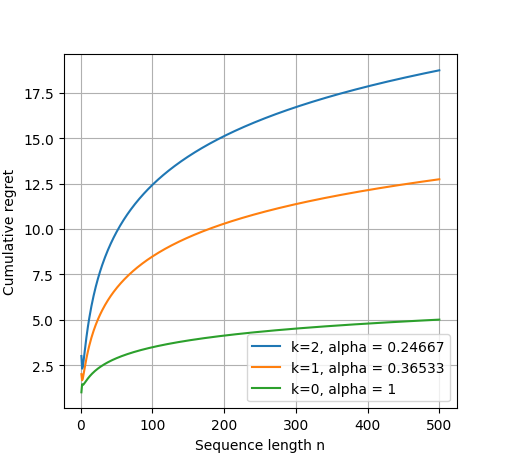}
        \caption{Cumulative regret incurred by EWOO for sequence $0^{n-1}1$ under 3 loss functions with $k \in \{0, 1, 2\}$, and varying sequence length.}
        \label{fig:4}
    \end{center}
\end{figure}

Here, we simulate the performance of the EWOO estimator under carefully selected binary outcome sequences, where we have empirically observed high regret. In particular, we vary $n$ from $1$ to $500$ and for each value of $n$, consider the outcome sequence $0^{n-1}1$. The best fixed distribution in hindsight for this sequence is computed numerically. Note that  the EWOO estimator involves an integration over the simplex; in our implementation, we numerically approximate it by discretizing the simplex. Fig.~\ref{fig:4} shows the average cumulative regret attained by the EWOO estimator for different values of $n$ under the sequence $0^{n-1}1$. We find that the regret scales logarithmically with the number of rounds and increases significantly with growing~$k.$

\section{Concluding Remarks}
In this work, we study the sequential probability assignment problem under a  loss function, which is a perturbation of the popular log loss and whose structure is motivated by the problem of ``twenty questions with a liar". We consider both the stochastic and adversarial setting, and provide upper bounds on the regret of the well-known add-$\beta$ and the EWOO algorithm respectively. Future research directions include proving lower bounds on the regret and considering other loss functions arising from guessing under different models for noise/lies. 
\remove{
In this work, we looked at a special loss function. We first obtained the logarithmic regret performance of the loss function in the stochastic setting. Although this result is non-trivial, the proof here is involved. The novel contribution of this paper is the regret performance of the loss function in the adversarial setting, which is logarithmic regret. This regret was obtained by first proving that the loss function falls under the class of $\alpha$-exp-concave function with $\alpha \in (0, 0.693)$, and thereby using the results of this class of functions in this setting.
}
\bibliographystyle{IEEEtran} 
\bibliography{references}

\begin{thebibliography}{00}

\bibitem{b11} Y. Dagan, Y. Filmus, D. Kane, and S. Moran, ``The entropy of lies: playing twenty questions with a liar'', ArXiv eprint 1811.02177, 2018.

\bibitem{b2} J. Jiao, K. Venkat, Y. Han, Yanjun, and T. Weissman, ``Maximum Likelihood Estimation of Functionals of Discrete Distributions'', IEEE Transactions on Information Theory, vol. 63, pp. 6774–6798, Oct 2017.

\bibitem{b3} C. Arnal, V. Cabannes, and V. Perchet, ``Mode Estimation with Partial Feedback'', Arxiv eprint 2402.13079, 2024.

\bibitem{b4} E. Hazan, ``Introduction to Online Convex Optimization'', Arxiv eprint 1909.05207, 2023.
\bibitem{b5}N.Merhav and M. Feder, ``Universal prediction", IEEE Transactions on Information Theory, vol. 44, no. 6, pp. 2124-2147, Jun. 2002. 

\bibitem{b6}
R.~E. Krichevsky and V.~K. Trofimov.
\newblock The performance of universal encoding.
\newblock {\em IEEE Transactions on Information Theory}, 27(2):199--207, 1981.

\bibitem{b7}
N.~Cesa-Bianchi and G.~Lugosi.
\newblock {\em Prediction, Learning, and Games}.
\newblock Cambridge University Press, 2006.

\bibitem{b8}
Y.~Lomnitz and M.~Feder,
``A Universal Probability Assignment for Prediction of Individual Sequences'', 2013 IEEE International Symposium on Information Theory.

\bibitem{b9}
M.~Feder, N.~Merhav and M.~Gutman, 
``Universal prediction of individual sequences'', 1991, 17th Convention of Electrical and Electronics Engineers in Israel, pages 223-226.

\bibitem{b10}
A.~Bhatt,
``Universal Prediction of m-ary Sequences'', 2023 IEEE International Symposium on Information Theory, pages 2236-2241.

\bibitem{b11} L. ~Zhou, and Alfred O. ~Hero.
\newblock {\em Twenty Questions with Random Error}.
\newblock Foundations and Trends in Communications and Information Theory, Vol. 22, pages 394–604, 2006.

\end{thebibliography}

\remove{

}
\section{Proofs}

\subsection{Proof of Theorem 1}
\begin{proof}

For each $t\in[n]$, the add-beta estimator $p_t$, as specified in \eqref{eqn:addbeta}, is
\begin{equation}
p_t(j)
=
\frac{m\beta}{t-1+m\beta}\frac{1}{m}
+
\frac{t-1}{t-1+m\beta}\bar p_{j,t-1}.
\end{equation}

The expected cumulative regret is
\begin{align}
R_n^S
&\triangleq
\sum_{t=1}^n
E_{y_t}
E_{y_{\tau}|_{\tau=1}^{t-1}}
\Bigg[
\log\frac{p(y_t)}{p_t(y_t)}
+
k\log
\left(
\frac{\log\!\left(\frac{1}{p_t(y_t)}+c\right)}
{\log\!\left(\frac{1}{p(y_t)}+c\right)}
\right)
\Bigg]
\nonumber\\
&=
\sum_{t=1}^{n}\sum_{j=1}^{m}
p(j)
E_{y_{\tau}|_{\tau=1}^{t-1}}
\Bigg[
\log\frac{p(j)}{p_t(j)}
+
k
\log
\left(
\frac{\log\!\left(\frac{1}{p_t(j)}+c\right)}
{\log\!\left(\frac{1}{p(j)}+c\right)}
\right)
\Bigg].
\end{align}

Define
\begin{equation}
r_t(j)
=
\log\frac{p(j)}{p_t(j)}
+
k
\log
\left(
\frac{\log\!\left(\frac{1}{p_t(j)}+c\right)}
{\log\!\left(\frac{1}{p(j)}+c\right)}
\right).
\end{equation}

Thus
\begin{align}
R_n^S
&=
\sum_{t=1}^{n}\sum_{j=1}^{m}
p(j)
E_{y_{\tau}|_{\tau=1}^{t-1}}[r_t(j)]
\nonumber\\
&=A+B,
\end{align}
where
\begin{align}
A
&=
\sum_{t=1}^{n}\sum_{j=1}^{m}
p(j)
E_{y_{\tau}|_{\tau=1}^{t-1}}
\big[
\mathbb{I}(N_{j,t-1}=0)r_t(j)
\big],
\\
B
&=
\sum_{t=1}^{n}\sum_{j=1}^{m}
p(j)
E_{y_{\tau}|_{\tau=1}^{t-1}}
\big[
\mathbb{I}(N_{j,t-1}\ge1)r_t(j)
\big].
\end{align}

Here $\mathbb{I}(\mathcal{E})$ denotes the indicator of event $\mathcal{E}$.

We first bound $A$ and then $B$.

\subsubsection{Upper bound on $A$}

Recall that $N_{j,t-1}\sim\mathrm{Bin}(t-1,p(j))$.
If $N_{j,t-1}=0$ then
\[
p_t(j)=\frac{m\beta}{t-1+m\beta}\frac{1}{m}.
\]

Hence
\begin{align}
&E_{y_{\tau}|_{\tau=1}^{t-1}}
\big[
\mathbb{I}(N_{j,t-1}=0)r_t(j)
\big]
\nonumber\\
&=
(1-p(j))^{t-1}\Bigg[
\log\frac{(t-1+m\beta)p(j)}{\beta}
\nonumber\\
&\quad
+
k
\log
\left(
\frac{
\log\!\left(\frac{t-1+m\beta}{\beta}+c\right)
}{
\log\!\left(\frac{1}{p(j)}+c\right)
}
\right)
\Bigg]
.
\end{align}

Therefore
\begin{align}
A
&=
\sum_{t=1}^{n}\sum_{j=1}^{m}
p(j)(1-p(j))^{t-1}
\nonumber\\
&\quad
\Bigg[
\log\!\left(p(j)\!\left(\frac{t-1}{\beta}+m\right)\right)
+
k\log\log\!\left(\frac{t-1}{\beta}+m+c\right)
\nonumber\\
&\qquad
-
k\log\log\!\left(\frac{1}{p(j)}+c\right)
\Bigg].
\end{align}

Let
\[
p_{\max}=\max_j p(j),\qquad
p_{\min}=\min_j p(j).
\]

Then
\begin{align}
A
&\le
\sum_{t=1}^{n}
(1-p_{\min})^{t-1}
\nonumber\\
&\quad
\Bigg[
\log\!\left(\frac{t-1}{\beta}+m\right)
-
H(p)
+
k\log\log\!\left(\frac{t-1}{\beta}+m+c\right)
\nonumber\\
&\qquad
-
k\log\log\!\left(\frac{1}{p_{\max}}+c\right)
\Bigg].
\label{Eqn:A}
\end{align}
Since $0 < (1-p_{\min}) < 1$, the factor $(1-p_{\min})^{t-1}$
decays exponentially in $t$. The remaining terms inside the
brackets grow at most logarithmically in $t$. Since the series
\[
\sum_{t=1}^{\infty}
(\log t)(1-p_{\min})^{t}
\]
converges, the above summation is bounded by a
constant independent of $n$, and we conclude that
\[
A = O(1).
\]
\subsubsection{Upper bound on $B$}

Recall
\[
p_t(j)=
\frac{m\beta}{t-1+m\beta}\frac{1}{m}
+
\frac{t-1}{t-1+m\beta}\bar p_{j,t-1}.
\]

If $N_{j,t-1}\ge1$ then $\bar p_{j,t-1}>0$.

\begin{align}
\mathbb{I}(N_{j,t-1}\ge1)r_t(j)
&=
\mathbb{I}(N_{j,t-1}\ge1)
\Bigg[
-\log\frac{p_t(j)}{p(j)}
\nonumber\\
&\hspace{-.5in}+
k
\log
\Bigg(
1+
\frac{
\log\!\left(\frac{1}{p_t(j)}+c\right)
-
\log\!\left(\frac{1}{p(j)}+c\right)
}{
\log\!\left(\frac{1}{p(j)}+c\right)
}
\Bigg)
\Bigg].
\end{align}

Using convexity of $-\log(\cdot)$ and Jensen's inequality, we have 
\begin{align}
-\log\frac{p_t(j)}{p(j)}
&\le
-\frac{m\beta}{t-1+m\beta}
\log\frac{1}{mp(j)}
\nonumber\\
&\quad
-
\frac{t-1}{t-1+m\beta}
\log\frac{\bar p_{j,t-1}}{p(j)}.
\end{align}

Using $\ln(1+x)\le x$  and $\log_2(1+x)\le x/\ln2$ for $x>-1$, we obtain
\begin{align}
\mathbb{I}(N_{j,t-1}\ge1)r_t(j)
&\le
\mathbb{I}(\bar p_{j,t-1}>0)
\Bigg[
-\frac{m\beta}{t-1+m\beta}\log\frac{1}{mp(j)}
\nonumber\\
&\quad
-
\frac{t-1}{t-1+m\beta}
\log\frac{\bar p_{j,t-1}}{p(j)}
\nonumber\\
&\quad
+
\frac{k}{(\ln2)\ln(\frac{1}{p(j)}+c)}
\frac{\frac{1}{p_t(j)}-\frac{1}{p(j)}}
{\frac{1}{p(j)}+c}
\Bigg].
\end{align}

Using convexity of $1/x$,
\[
\frac{1}{p_t(j)}
\le
\frac{m\beta}{t-1+m\beta}m
+
\frac{t-1}{t-1+m\beta}
\frac{1}{\bar p_{j,t-1}}.
\]

Substituting yields
\[
\mathbb{I}(N_{j,t-1}\ge1)r_t(j)=C+D
\]

where

\begin{align}
C
&=
\mathbb{I}(\bar p_{j,t-1}>0)
\Bigg[
-\frac{m\beta}{t-1+m\beta}\log\frac{1}{mp(j)}
\nonumber\\
&\quad
+
\frac{k}{(\ln2)\ln(\frac{1}{p(j)}+c)}
\frac{m\beta}{t-1+m\beta}
\frac{m-\frac{1}{p(j)}}{\frac{1}{p(j)}+c}
\Bigg],
\end{align}

\begin{align}
D
&=
\mathbb{I}(\bar p_{j,t-1}>0)
\Bigg[
-\frac{t-1}{t-1+m\beta}
\log\frac{\bar p_{j,t-1}}{p(j)}
\nonumber\\
&\quad
+
\frac{k}{(\ln2)\ln(\frac{1}{p(j)}+c)}
\frac{t-1}{t-1+m\beta}
\frac{\frac{1}{\bar p_{j,t-1}}-\frac{1}{p(j)}}
{\frac{1}{p(j)}+c}
\Bigg].
\end{align}

Taking expectations,
\begin{equation}
E[\mathbb{I}(N_{j,t-1}\ge1)r_t(j)]
=
E[C]+E[D].
\label{Eqn:CplusD}
\end{equation}

Bounding $E[C]$ gives
\begin{align}
E[C]
&\le
\frac{m\beta}{t-1+m\beta}\log m
\nonumber\\
&\quad
+
\frac{k}{(\ln2)\ln(\frac{1}{p(j)}+c)}
\frac{m^2\beta}{(t-1+m\beta)(\frac{1}{p(j)}+c)}.
\label{Eqn:C}
\end{align}

We next bound the expectation of $D$.

\begin{proposition}
\label{Prop1}
For every $t$ and $j$, we have 
\begin{align}
E[D]
&\le
\frac{2(t-1)}{t-1+m\beta}
e^{-\frac{(t-1)p(j)}{10}}\log(t-1)
\nonumber\\
&\quad
+
\frac{4}{(\ln2)(t-1+m\beta)p(j)}
\nonumber\\
&\quad
+
\frac{k}{(\ln2)\ln\!\left(\frac{1}{p(j)}+c\right)}
\frac{1}{\frac{1}{p(j)}+c}
\nonumber\\
&\quad
\Bigg[
\frac{(t-1)^2}{t-1+m\beta}e^{-\frac{(t-1)p(j)}{8}}
+
\frac{8}{p^2(j)(t-1+m\beta)}
\Bigg].
\label{Eqn:D}
\end{align}
\label{prop}
\end{proposition}

We relegate the proof of Proposition~\ref{prop} to the end of this section.
Substituting the bounds from \eqref{Eqn:C} and \eqref{Eqn:D} into
\eqref{Eqn:CplusD} and defining
\begin{equation*}
l=\left(\frac{1}{p_{max}}+c\right)
\ln\!\left(\frac{1}{p_{max}}+c\right),
\qquad
s=1-m\beta,
\end{equation*}
we obtain
\begin{align}
E_{y_{\tau}|_{\tau=1}^{t-1}}
\!\left[
\mathbb{I}(N_{j,t-1}\ge1) r_t(j)
\right]
&\le
\frac{1}{t-s}
\Bigg[
m\beta\log m
+
\frac{4}{(\ln2)p(j)}
\nonumber\\
&+
\frac{k}{(\ln2)l}
\left(
m^{2}\beta+\frac{8}{p^{2}(j)}
\right)
\Bigg]
+E ,
\end{align}
where
\begin{align}
E
&=
\frac{t-1}{t-s}
\Bigg[
\frac{k}{(\ln2)l}(t-1)
e^{-\frac{(t-1)p(j)}{8}}
+
2e^{-\frac{(t-1)p(j)}{10}}
\log(t-1)
\Bigg].
\end{align}
Recall that
\[
B
=
\sum_{t=1}^{n}
\sum_{j=1}^{m}
p(j)
E_{y_{\tau}|_{\tau=1}^{t-1}}
\!\left[
\mathbb{I}(N_{j,t-1}\ge1) r_t(j)
\right].
\]

Noting that for $t = 1$ and any $j$, $\mathbb{I}(N_{j,t-1}\ge1) = 0$, and then substituting the above bound yields
\begin{align}
B
&\le
\sum_{t=2}^{n}
\sum_{j=1}^{m}
p(j)E
\nonumber\\
&\quad+
\sum_{t=2}^{n}
\sum_{j=1}^{m}
p(j)
\frac{1}{t-s}
\Bigg[
m\beta\log m
+
\frac{4}{(\ln2)p(j)}
\nonumber\\
&\qquad+
\frac{k}{(\ln2)l}
\left(
m^{2}\beta+\frac{8}{p^{2}(j)}
\right)
\Bigg].
\label{Eqn:B}
\end{align}
Consider the first term. Since $s=1-m\beta$ is a constant, we have $\frac{t-1}{t-s}=O(1)$. Hence
\[
E = O\!\left((t-1)^2 e^{-\frac{(t-1)p(j)}{8}} + \log(t-1)e^{-\frac{(t-1)p(j)}{10}}\right).
\]
Since exponential decay dominates polynomial and logarithmic growth, both
\[
\sum_{t=1}^{\infty} t^2 e^{-cp(j)t}
\quad\text{and}\quad
\sum_{t=1}^{\infty} \log(t)e^{-cp(j)t}
\]
converge for any $c>0$. Therefore
\[
\sum_{t=1}^{n}\sum_{j=1}^{m} p(j)E = O(1).
\]

Further, using
\[
\sum_{j=1}^{m}p(j)=1,
\qquad
p(j)\ge p_{min},
\]
the bound in \eqref{Eqn:B} becomes
\begin{align}
B
&\le
O(1) 
\nonumber\\
&+
\left(
m\beta\log m
+
\frac{4m}{\ln2}
+
\frac{k}{(\ln2)l}
\left(
m^{2}\beta+\frac{8m}{p_{min}}
\right)
\right)
\sum_{t=2}^{n}
\frac{1}{t-s}.
\end{align}

Since $s=1-m\beta$, we have
$t-s=t-1+m\beta$. Thus
\begin{align}
\sum_{t=2}^{n}\frac{1}{t-s}
=
\sum_{t=2}^{n}\frac{1}{t-1+m\beta}
\le
\sum_{t=2}^{n}\frac{1}{t-1+\lfloor m\beta\rfloor}.
\end{align}

Using the standard harmonic bound,
\begin{align}
\sum_{t=2}^{n}
\frac{1}{t-1+\lfloor m\beta\rfloor}
\le
\ln\!\left(
\frac{n-1}{\lfloor m\beta\rfloor}+1
\right).
\end{align}

Therefore
\begin{align}
B
&\le
O(1)
\nonumber\\
&\quad+
\Bigg(
m\beta\log m
+
\frac{4m}{\ln2}
+
\frac{k}{(\ln2)l}
\left(
m^{2}\beta+\frac{8m}{p_{min}}
\right)
\Bigg)
\nonumber\\
&\qquad\times
\ln\!\left(
\frac{n-1}{\lfloor m\beta\rfloor}+1
\right).
\end{align}

Finally, combining with the bound on $A$ in \eqref{Eqn:A} and noting that
\[
A+\sum_{t=1}^{n}\sum_{j=1}^{m}p(j)E=O(1),
\]
we obtain
\begin{align}
R_n^S
&=
\sum_{t=1}^{n}\sum_{j=1}^{m}
p(j)
E_{y_{\tau}|_{\tau=1}^{t-1}}
[r_t(j)] = A + B
\nonumber\\
&\le
\Bigg(
m\beta\log m
+
\frac{4m}{\ln2}
+
\frac{k}{(\ln2)l}
\left(
m^{2}\beta+\frac{8m}{p_{min}}
\right)
\Bigg)
\nonumber\\
&\qquad\times
\ln\!\left(
\frac{n-1}{\lfloor m\beta\rfloor}+1
\right)
+O(1).
\end{align}

This completes the proof of Theorem~1.
\end{proof}

\subsection*{Proof of Proposition~\ref{Prop1}}
\begin{proof}
Define
\[
\alpha_j
=
\frac{k}
{(\ln2)\ln\!\left(\frac1{p(j)}+c\right)}
\cdot
\frac{1}{\frac1{p(j)}+c}.
\]

\begin{align}
E[D]
&=
\frac{t-1}{t-1+m\beta}
E\!\left[
\mathbb{I}(\bar p_{j,t-1}>0)
\left(
-\log\frac{\bar p_{j,t-1}}{p(j)}
\right)
\right]
\nonumber\\
&+
\frac{t-1}{t-1+m\beta}
\cdot \alpha_j
\cdot
E\!\left[
\mathbb{I}(\bar p_{j,t-1}>0)
\left(
\frac1{\bar p_{j,t-1}}-\frac1{p(j)}
\right)
\right].
\label{Eqn:ED}
\end{align}

From Lemmas 4 and 5 of \cite{b3},
\begin{align}
E\!\left[
-\mathbb{I}(\bar p_{j,t-1}>0)
\log\frac{\bar p_{j,t-1}}{p(j)}
\right]
&\le
2e^{-\frac{(t-1)p(j)}{10}}\log(t-1)
\nonumber\\
&\quad+
\frac{4}{(\ln2)(t-1)p(j)} .
\label{Eqn:Bound1}
\end{align}

Next,
\begin{align}
&
E\!\left[
\mathbb{I}(\bar p_{j,t-1}>0)
\left(
\frac{1}{\bar p_{j,t-1}}-\frac{1}{p(j)}
\right)
\right]
\nonumber\\
&=
E\!\Bigg[\mathbb{I}(\bar p_{j,t-1}>0)
\mathbb{I}\!\left(
\left|\frac{\bar p_{j,t-1}-p(j)}{p(j)}\right|\ge \tfrac12
\right)
\left(
\frac{1}{\bar p_{j,t-1}}-\frac{1}{p(j)}
\right)
\Bigg]
\nonumber\\
&+
E\!\Bigg[\mathbb{I}(\bar p_{j,t-1}>0)
\mathbb{I}\!\left(
\left|\frac{\bar p_{j,t-1}-p(j)}{p(j)}\right|< \tfrac12
\right)
\left(
\frac{1}{\bar p_{j,t-1}}-\frac{1}{p(j)}
\right)
\Bigg].
\label{Eqn:Part2}
\end{align}

Define
\[
f(N) = \frac{t-1}{N} - \frac{1}{p(j)}.
\]

Since $\bar p_{j,t-1} = \frac{N_{j,t-1}}{t-1}$, we have for the first term:
\begin{align}
&
E\!\Bigg[\mathbb{I}(\bar p_{j,t-1}>0)
\mathbb{I}\!\left(
\left|\frac{\bar p_{j,t-1}-p(j)}{p(j)}\right|\ge \tfrac12
\right)
f(N_{j,t-1})
\Bigg]
\nonumber\\
&\overset{(a)}{\le}
E\!\Bigg[
\mathbb{I}\!\left(1 \le
N_{j,t-1}\le \tfrac12 (t-1)p(j)
\right)
f(N_{j,t-1})
\Bigg]
\nonumber\\
&\overset{(b)}{\le}
\left(t-1-\frac{1}{p(j)}\right)
\mathbb{P}\!\left(
N_{j,t-1}\le \tfrac12 (t-1)p(j)
\right)
\nonumber\\
&\overset{(c)}{\le}
\left(t-1-\frac{1}{p(j)}\right)
e^{-\frac{(t-1)p(j)}{8}}\nonumber\\
&\le \left(t-1\right)e^{-\frac{(t-1)p(j)}{8}}
\label{Eqn:Bound2}
\end{align}
where:
\begin{itemize}
\item[(a)] uses that the events 
\[ \bar p_{j,t-1} > 0, 
\left|\frac{\bar p_{j,t-1}-p(j)}{p(j)}\right|\ge \frac{1}{2}
\]
imply either 
\(
1 \le N_{j,t-1}\le \tfrac12 (t-1)p(j)
\)
or 
\(
N_{j,t-1}\ge \tfrac32 (t-1)p(j)
\),
and the latter is dropped since $f(N)$ is negative in that case.

\item[(b)] uses the bound 
\[
f(N_{j,t-1}) \le t-1-\frac{1}{p(j)}
\]
for $N_{j,t-1}\ge 1$.

\item[(c)] follows from the Chernoff bound for 
\(
N_{j,t-1}\sim \mathrm{Bin}(t-1,p(j))
\).
\end{itemize}
For the second term, define
\[
X := \frac{\bar p_{j,t-1}-p(j)}{p(j)}.
\]
and note that 
\[
\frac{1}{\bar p_{j,t-1}}
=
\frac{1}{p(j)}\cdot\frac{1}{1+X}.
\]

Thus,
\begin{align}
&
E\!\Bigg[
\mathbb{I}(\bar p_{j,t-1}>0)
\mathbb{I}(|X|<\tfrac12)
\left(
\frac{1}{\bar p_{j,t-1}}-\frac{1}{p(j)}
\right)
\Bigg]
\nonumber\\
&\le
\frac{1}{p(j)}
E\!\Bigg[
\mathbb{I}(|X|<\tfrac12)
\left(
\frac{1}{1+X}-1
\right)
\Bigg]
\nonumber\\
&\overset{(d)}{\le}
\frac{1}{p(j)}
E\!\Big[
-X + 8X^2
\Big]
\nonumber\\
&\overset{(e)}{=}
\frac{1}{p(j)}
E\!\left[
8X^2
\right]
\nonumber\\
&\overset{(f)}{=}
\frac{8}{p(j)}
\cdot
\frac{\mathrm{Var}(\bar p_{j,t-1})}{p^2(j)}
\nonumber\\
&\overset{(g)}{=}
\frac{8(1-p(j))}
{(t-1)p^2(j)}\nonumber\\
&\le 
\frac{8}
{(t-1)p^2(j)}
\label{Eqn:Bound}
\end{align}
where:
\begin{itemize}
\item[(d)] uses the inequality 
\(
(1+x)^{-1} \le 1 - x + 8x^2
\)
valid for $|x|\le \tfrac12$;

\item[(e)] follows since $E[\bar p_{j,t-1}] = p_j$, and hence 
\(
E[X]=0
\);
\item[(f)] uses
\[
X = \frac{\bar p_{j,t-1}-p(j)}{p(j)}
\quad\Rightarrow\quad
E[X^2]
=
\frac{\mathrm{Var}(\bar p_{j,t-1})}{p^2(j)}, \mbox{ and }
\]
\item[(g)] follows since $\mathrm{Var}(\bar p_{j,t-1}) = p_j(1-p_j)/ (t-1)$.
\end{itemize}
Plugging the bounds in \eqref{Eqn:Bound2} and \eqref{Eqn:Bound} in the expression in \eqref{Eqn:Part2}, we get
\begin{align}
&
E\!\left[
\mathbb{I}(\bar p_{j,t-1}>0)
\left(
\frac{1}{\bar p_{j,t-1}}-\frac{1}{p(j)}
\right)
\right]
\nonumber\\
&\le
\left(t-1\right)
e^{-\frac{(t-1)p(j)}{8}}
+
\frac{8}
{(t-1)p^2(j)}.
\label{Eqn:Bound3}
\end{align}

Finally, substituting the bounds from \eqref{Eqn:Bound1} and \eqref{Eqn:Bound3} into the expression of $E[D]$ from \eqref{Eqn:ED}:
\begin{align}
E[D]
&\le
\frac{2(t-1)}{t-1+m\beta}
e^{-\frac{(t-1)p(j)}{10}}
\log(t-1)
\nonumber\\
&+
\frac{4}{(\ln2)(t-1+m\beta)p(j)}
\nonumber\\
&+
\alpha_j
\Bigg(
\frac{(t-1)^2}{t-1+m\beta}
e^{-\frac{(t-1)p(j)}8}
+
\frac{8}
{(t-1+m\beta)p^2(j)}
\Bigg).
\nonumber
\end{align}
This completes the proof of Proposition~\ref{Prop1}.
\end{proof}

\subsection{Proof of Theorem~\ref{Thm2}}
\begin{proof}
The loss in round $t \in [n]$ is given by $l(p_{t}, y_{t}) = f(p_{t}(y_{t}))$. This can also be expressed as $f(<p_t, e_t>)$ with the adversary (oracle) choosing a binary $m$-length vector $e_{t}$ defined as
\begin{align*}
e_{t}(j) =
  \begin{cases}
    1 \ \ \ \ \ \ \ \ \text{if} \ y_{t} = j\\
    0 \ \ \ \ \ \ \ \ \text{if} \ y_{t} \neq j
  \end{cases}
\end{align*} 
to generate symbol $y_{t} \in \{1, 2, \cdots m\}$. Thus, demonstrating the exp-concavity of the loss function $l(\cdot, y_t)$ is equivalent to showing that the function $f(x) = \log\frac{1}{x} + k \log\log(\frac{1}{x} + c)$ is exp-concave for $x \in (0,1)$. To do this, we study the concavity of the function
\begin{align*}
g(x) = e^{-\alpha f(x)} 
&= \exp\!\left(-\alpha \left(\log \frac{1}{x} + k \log \log \left(\frac{1}{x} + c\right)\right)\right)\\
&= x^{\alpha \log_2 e}\left(\log \left(\frac{1}{x} + c\right)\right)^{-\alpha k \log_2 e}.
\end{align*}
\remove{
Then
\begin{align*}
g(x)
&= \exp\!\left(\alpha \log x - \alpha k \log \log \left(\frac{1}{x} + c\right)\right) \\

\end{align*}
}
Let \(a = \alpha \log_2 e\). Then
\begin{equation}
g(x)
=
\left\{
\frac{x}{\left(\log \left(\frac{1}{x} + c\right)\right)^k}
\right\}^{a}.
\end{equation}

We study concavity of \(g(x)\) over \(x \in (0,1)\). Below are the expressions for the first and second derivative of $g(x)$.
\begin{equation}
g'(x)
=
\frac{a x^{a-1}}{\left(\log \left(\frac{1}{x} + c\right)\right)^{ka}}
+
\frac{k a x^{a-2}}
{(\ln 2)\left(\frac{1}{x} + c\right)
\left(\log \left(\frac{1}{x} + c\right)\right)^{ka+1}}.
\end{equation}

\begin{align*}
g''(x)
&=
\frac{a(a-1)x^{a-2}}
{\left(\log \left(\frac{1}{x} + c\right)\right)^{ka}} \\
&\quad+
\frac{2ka(a-1)x^{a-3}}
{(\ln 2)\left(\frac{1}{x} + c\right)
\left(\log \left(\frac{1}{x} + c\right)\right)^{ka+1}} \\
&\quad+
\frac{ka x^{a-4}}
{(\ln 2)\left(\frac{1}{x} + c\right)^2
\left(\log \left(\frac{1}{x} + c\right)\right)^{ka+1}} \\
&\quad+
\frac{ka(ka+1)x^{a-4}}
{(\ln 2)^2\left(\frac{1}{x} + c\right)^2
\left(\log \left(\frac{1}{x} + c\right)\right)^{ka+2}}.
\end{align*}

Rearranging, we have
\begin{equation}
g''(x)
=
\frac{aG(x)}
{x^{4-a}(\ln 2)^2\left(\frac{1}{x} + c\right)^2
\left(\log \left(\frac{1}{x} + c\right)\right)^{ka+2}},
\end{equation}
where
\begin{align*}
G(x)
&=
k^{2}a + k + k\ln\left(\frac{1}{x}+c\right) \\
&-
(1-a)(1+xc)
\Big[
(1+xc)\left(\ln\left(\frac{1}{x}+c\right)\right)^2 \\
&\hspace{1.7in}+ 2k\ln\left(\frac{1}{x}+c\right)
\Big].
\end{align*}

Since the denominator is positive,
\begin{equation*}
g''(x) \leq 0 \iff G(x) \leq 0.
\end{equation*}
which in turn implies that the requirement for $g(x)$ being concave is given by 
\begin{align*}
k^{2}a + k + k\ln \left(\frac{1}{x}+c\right)
&\leq
(1-a)(1+xc)^2\left(\ln\left(\frac{1}{x}+c\right)\right)^2 \\
&\quad+
2k(1-a)(1+xc)\ln\left(\frac{1}{x}+c\right), 
\end{align*}
or equivalently, we need that for all $x \in (0,1]$, 
\begin{equation}
a \leq
\frac{(1+xc)^2(\ln(\frac{1}{x}+c))^2
+ k(1+2xc)\ln(\frac{1}{x}+c) - k}
{(1+xc)^2(\ln(\frac{1}{x}+c))^2
+ 2k(1+xc)\ln(\frac{1}{x}+c) + k^2}.
\label{Eqn:requirement}
\end{equation}
Recall that we defined 
\begin{equation*} 
    a^{*}\! = \!\underset{x \in (0, 1]}{\textnormal{min}}\ \frac{(1+xc)^{2}(\ln(\frac{1}{x}+c))^{2} + k(1+2xc)\ln(\frac{1}{x}+c) - k}{(1+xc)^{2}(\ln(\frac{1}{x}+c))^{2} + 2k(1+xc)\ln(\frac{1}{x}+c)+k^{2}}.
\end{equation*} 
Here \(a^* \in (0,1]\), and since the requirement in \eqref{Eqn:requirement} is clearly satisfied for all $a \le a^*$, we have that the function \(f(x)\) is \(a / \log_2 e\)-exp-concave for all \(a \le a^*\).

Given the complex description of $a^*$, next we provide a simpler value $a^{**}$ which is smaller than $a^*$. Thus, the function $f(x)$ is also exp-concave for all $a \le a^{**}$. We start with the r.h.s. in \eqref{Eqn:requirement} and note that 
\begin{align*}
&\frac{(1+xc)^2(\ln(\frac{1}{x}+c))^2
+ k(1+xc)\ln(\frac{1}{x}+c) - k}
{(1+xc)^2(\ln(\frac{1}{x}+c))^2
+ 2k(1+xc)\ln(\frac{1}{x}+c) + k^2} \\
&=
1 - \frac{k((1+xc)\ln(\frac{1}{x}+c) + k) + k}
{((1+xc)\ln(\frac{1}{x}+c) + k)^2} \\
&=
1 - \frac{k}{(1+xc)\ln(\frac{1}{x}+c) + k}
- \frac{k}{((1+xc)\ln(\frac{1}{x}+c) + k)^2}.
\end{align*}

Thus,
\begin{equation}
a^*
\ge
1 - \frac{k}{\gamma(c)+k}
- \frac{k}{(\gamma(c)+k)^2}
= a^{**},
\end{equation}
where recall that $\gamma(c) = \underset{x \in (0, 1]}{\textnormal{min}}\left[(1+xc)\ln\big(\frac{1}{x}+c\big)\right]$. 
\remove{
Next, we provide an upper bound on the gap between the two constants $a^*$ and $a^{**}$. We have
\begin{align*}
a^* - a^{**}
&{\color{red}=}
\min_{x \in (0,1]}
\frac{kxc\ln(\frac{1}{x}+c)}
{\big((1+xc)\ln(\frac{1}{x}+c) + k\big)^2} \\
&\le
\min_{x}
\frac{1}{
\frac{xc\ln(\frac{1}{x}+c)}{k}
+ \frac{k}{xc\ln(\frac{1}{x}+c)} + 2}
\le \frac{1}{4}
\end{align*}
where the last inequality follows from the AM-GM inequality. 
}

Next, we  show an explicit lower bound on $a^{**}$ (and thus $a^*$ as well). Note that for \(c \ge 1\), we have 
\begin{align*}
\gamma(c) &= \underset{x \in (0, 1]}{\textnormal{min}}(1+xc)\ln\bigg(\frac{1}{x}+c\bigg)\\
&\ge
\min_{x \in (0,1]} (1+x)\ln\left(\frac{1}{x}+1\right)
= 2\ln 2.
\end{align*}
\remove{
\begin{equation*}
(1+xc)\ln\left(\frac{1}{x}+c\right)
\ge
(1+x)\ln\left(\frac{1}{x}+1\right),
\end{equation*}
so that
}
Thus, 
\begin{align*}
a^{**}
&=
1 - \frac{k}{\gamma(c)+k}
\left(1 + \frac{1}{\gamma(c)+k}\right) \\
&\ge
1 - \frac{k}{2\ln 2 + k}
\left(1 + \frac{1}{2\ln 2 + k}\right) \\
&=
\frac{(2\ln 2)^2 + (2\ln 2 - 1)k}
{(2\ln 2 + k)^2}.
\end{align*}
Next we want to argue that 
\begin{equation}
\frac{(2\ln 2)^2 + (2\ln 2 - 1)k}
{(2\ln 2 + k)^2} \ge \frac{(2\ln 2 - 1)}{k}, 
\end{equation}
or equivalently that 
\begin{align*}
 \frac{k\big((2\ln 2)^2 + (2\ln 2 - 1)k\big) - (2\ln 2 - 1)(2\ln 2 + k)^2}
{k(2\ln 2 + k)^2} &\ge 0
\end{align*}
Thus, it is enough to show that the numerator is nonnegative. Expanding it, we get
\begin{align}
&k\big((2\ln 2)^2 + (2\ln 2 - 1)k\big) - (2\ln 2 - 1)(2\ln 2 + k)^2 \nonumber\\
&= k(2\ln 2)^2 + (2\ln 2 - 1)k^2 \nonumber\\
&\quad - (2\ln 2 - 1)\big((2\ln 2)^2 + 2(2\ln 2)k + k^2\big) \nonumber\\
&= k(2\ln 2)^2 - 2(2\ln 2)(2\ln 2 - 1)k - (2\ln 2 - 1)(2\ln 2)^2 \nonumber\\
&= (2\ln 2)\Big( (2-2\ln 2)k - (2\ln 2)(2\ln 2 - 1)\Big).
\end{align}
Hence it suffices to verify that
\begin{equation}
(2-2\ln 2)k - (2\ln 2)(2\ln 2 - 1) \ge 0.
\end{equation}
Since
\begin{equation}
\frac{(2\ln 2)(2\ln 2 - 1)}{2-2\ln 2} < 1,
\end{equation}
the above inequality holds for every $k \ge 1$. Therefore,
\begin{equation}
\frac{(2\ln 2)^2 + (2\ln 2 - 1)k}
{(2\ln 2 + k)^2} \ge \frac{(2\ln 2 - 1)}{k}.
\end{equation}
It follows that
\begin{equation}
a^{*} \ge \frac{2\ln 2 - 1}{k}.
\end{equation}

\remove{
{\color{red}
Thus \(a = O\!\left(\frac{1}{k}\right)\), and moreover
\begin{equation}
a^{**}
\ge
\left(\frac{2\ln 2}{2\ln 2 + k}\right)^2 > 0,
\end{equation}
which implies $a^* > \tfrac{1}{4}$.}
}
\end{proof}

\end{document}